\documentclass[final]{alt2026} 

\title[Distributional Soft Bellman Operator under the Cram\'er Geometry]{Distributional Soft Bellman Operator under the Cram\'er Geometry}
\usepackage{times}
\usepackage{amsmath, amssymb, bm, mathtools}
\usepackage{bbm}
\usepackage{mathrsfs}

\usepackage{graphicx}
\usepackage{booktabs}
\usepackage{array}
\usepackage{caption}
\usepackage{subcaption}

\usepackage{enumitem}
\usepackage{xcolor}
\usepackage{hyperref}
\hypersetup{
    colorlinks=true,
    linkcolor=blue,
    citecolor=blue,
    urlcolor=blue
}

\usepackage[a4paper,margin=1in]{geometry}

\newtheorem{assumption}[theorem]{Assumption}

\altauthor{\Name{Keru Wang}\textsuperscript{1} \Email{keru.wang@ucdconnect.ie} \\
       \Name{Yixin Deng}\textsuperscript{1} \Email {yixin.deng@ucdconnect.ie} \\
       \Name {Yao Lyu}\textsuperscript{2} \Email {lvyao@mail.tsinghua.edu.cn} \\
       \Name {Stephen Redmond}\textsuperscript{1}\thanks{Corresponding authors.} \Email {stephen.redmond@ucd.ie} \\
       \Name {Shengbo Eben Li}\textsuperscript{2,3}\footnotemark[1] \Email {lishbo@tsinghua.edu.cn} \\
       \addr \textsuperscript{1}School of Electrical and Electronic Engineering, University College Dublin, Dublin, Ireland\\
       \addr \textsuperscript{2}School of Vehicle and Mobility, Tsinghua University, Beijing, China\\
       \addr \textsuperscript{3}College of Artificial Intelligence, Tsinghua University, Beijing, China}

\makeatletter
\renewcommand*{\@jmlrproceedings}{}
\renewcommand*{\@jmlrpages}{}
\makeatother
\begin{document}

\maketitle

\begin{abstract}%
  Distributional soft policy iteration (DSPI) provides an important framework for combining distributional reinforcement learning (DRL) with maximum-entropy control, in which the policy evaluation step is governed by a distributional soft Bellman operator acting on entropy-regularised returns. Theoretical analysis of such an evaluation step requires a probability metric under which Bellman updates can be controlled, typically by showing that the operator contracts the distance between any two candidate return-distribution estimates. In this paper, we focus on the Cram\'er geometry, a cumulative distribution function (CDF)-based metric with an $L^2$ structure, and study whether the fixed-policy distributional soft Bellman operator has this contraction property and hence a unique fixed point under this metric. Working directly on an admissible CDF field domain, we formulate the CDF-level distributional soft Bellman operator, prove that it is a $\sqrt{\gamma}$-contraction, and obtain the corresponding unique fixed point together with convergent iterative policy evaluation. The CDF formulation also shows that this finite-Cram\'er-domain property follows from a uniform first-moment condition on the combined one-step reward entropy shift, rather than from separate uniform boundedness assumptions on the reward and entropy terms. We then transport the same evaluation problem to the spectral domain by conjugation, obtaining an equivalent Hilbert-space representation of the same decision process. Taken together, these results identify the Cram\'er-geometric Bellman fixed point associated with the policy-evaluation step of DSPI, providing a reference point for studying approximate critics, evaluation error, and critic-loss design in DSPI-style algorithms.
\end{abstract}

\begin{keywords}%
  Distributional Reinforcement Learning; Distributional Soft Bellman Operator; Cram\'er Metric; Spectral Realisation; Distributional Soft Policy Iteration.
\end{keywords}

\section{Introduction}
Distributional reinforcement learning (DRL) \citep{c51,drlbook_bellemare_distributional_2023} studies Bellman updates at the level of return distributions rather than only expectations. In the policy evaluation setting \citep{c51,drlbook_bellemare_distributional_2023,c51theory}, contraction-based analysis depends on the probability metric used to compare return distributions \citep{c51theory,cramer-analysis,drlstat}. In this work, we focus on the Cram\'er metric \citep{cramer1,cramer2,cramer-analysis}, a cumulative distribution function (CDF)-based metric under which distances between return distributions are expressed as \(L^2\) distances between CDFs. Recent work \citep{our1st} has further shown that, for standard, non-soft distributional policy evaluation under the Cram\'er geometry, the same evaluation problem admits an equivalent spectral Hilbert-space representation. Thus, under the Cram\'er geometry, standard distributional policy evaluation can be understood equivalently at the distributional, CDF, and spectral levels.

A parallel line of work studies maximum-entropy reinforcement learning (RL) \citep{todorov07,todorov09,toussaint09,rawlik12,neu17,geist19,haarnoja17,maxentrl1,sac1,sac2}, where entropy regularisation modifies the Bellman recursion and leads to the soft Bellman framework underlying soft policy iteration. Combining this viewpoint with DRL yields distributional soft policy iteration (DSPI) and related algorithmic formulations \citep{dsacDuan_2022,dsacma,dsactDuan_2025}, in which one evaluates entropy-regularised return distributions rather than only their expectations. In some practical DSPI or distributional soft actor-critic (DSAC) style algorithms, policy evaluation is implemented by fitting a parameterised distributional critic to a bootstrapped Bellman target distribution, often through a tractable distribution-matching loss such as a KL divergence. This choice is algorithmically convenient, especially for parametric critic families where the loss can be estimated or optimised directly from samples. However, tractability of the critic loss does not by itself imply that the associated Bellman evaluation step is non-expansive or contractive under the same geometry, so errors in the bootstrapped target may be propagated or amplified rather than systematically reduced by repeated critic updates. Although such methods have achieved strong empirical performance, DSAC has also been reported to suffer from instability in challenging high-dimensional control tasks \citep{dsacDuan_2022}; while such behaviour may have multiple algorithmic causes, the absence of a contraction guarantee for the underlying bootstrapped evaluation step is a natural theoretical issue to examine. Subsequent refinements, such as DSAC-T, introduce more sophisticated training mechanisms to improve empirical stability \citep{dsactDuan_2025}; however, these refinements do not by themselves guarantee that the underlying distributional soft Bellman evaluation step is contractive. Without such an evaluation-level guarantee, it is difficult to explain theoretically why iterative critic updates should converge to the policy's entropy-regularised return distribution, which is a prerequisite for analysing the subsequent policy improvement step. The Cram\'er geometry is a useful candidate for analysing the Bellman evaluation step because its CDF-$L^2$ structure turns the distributional Bellman update into a shift-and-rescale operation on CDFs, for which the discount rescaling yields a direct contraction estimate, while Its connection with the energy distance also suggests a practical sample-based way to estimate Cram\'er-type critic losses, avoiding the need to compute the Cram\'er distance through an explicit CDF integral and making such losses more readily implementable in DRL algorithms. Motivated by these considerations, this paper takes a fixed-policy evaluation perspective and studies whether the distributional soft Bellman operator admits a contraction and fixed-point theory under the Cram\'er geometry.

This paper studies the policy evaluation component of DSPI under the Cram\'er geometry. At the CDF level, the soft Bellman update can be written as a shifted and rescaled operator on CDF fields, while the Cram\'er metric is measured by the \(L^2\) distance between those CDFs. This metric-compatible formulation yields a direct proof of the \(\sqrt{\gamma}\)-contraction. It also makes explicit the condition needed for the entropy-regularised Bellman target to remain in the finite-Cram\'er-distance domain: a uniform first-moment condition on the combined one-step soft shift, rather than separate uniform boundedness assumptions on the reward and entropy terms \citep{geist19,sac1,dsacDuan_2022}. Building on this CDF-level formulation, we then transport the same evaluation problem to the spectral domain by conjugation \citep{operator-conjugate}, obtaining an equivalent Hilbert-space representation of the same dynamics.

Our contributions are as follows.
\begin{itemize}[leftmargin=2em]
    \item We formulate the distributional soft Bellman operator directly at the CDF level on an admissible CDF field domain associated with the Cram\'er geometry, and show that this formulation is consistent with the underlying distribution-level soft Bellman recursion.

    \item We prove a \(\sqrt{\gamma}\)-contraction and the resulting fixed-point and convergence guarantees for the CDF-level distributional soft Bellman operator, with domain preservation ensured by a first-moment condition on the combined one-step reward-entropy shift rather than by separate uniform boundedness assumptions on the reward and entropy terms \citep{geist19,sac1,dsacDuan_2022}.

    \item We transport the same evaluation problem to the spectral domain by conjugation, obtaining an exact Hilbert-space realisation of the CDF-level dynamics. This provides an equivalent spectral representation of the same soft evaluation problem and identifies the exact Cram\'er-geometric Bellman target that the policy-evaluation step of DSPI would attempt to estimate.
\end{itemize}

The remainder of the paper is organised as follows. Section~\ref{sec:related-work} positions the paper relative to the relevant literature. Section~\ref{sec:preliminaries} introduces the probabilistic setup, the CDF field domain, and the spectral transport used later. Section~\ref{sec:cdf-soft-evaluation} develops the CDF-level formulation and soft policy evaluation theory of the distributional soft Bellman operator under the Cram\'er geometry. Section~\ref{sec:spectral-soft} gives the spectral realisation of the same dynamics. Section~\ref{sec:dspi-perspective} explains how the resulting evaluation theory fits into the broader framework of DSPI. Section~\ref{sec:discussion} discusses the scope of the analysis and possible extensions to projected or approximate soft distributional critics, and Section~\ref{sec:conclusion} concludes.

\section{Related work}
\label{sec:related-work}
In DRL, early work already studied return distributions through non-parametric and parametric return-density estimation \citep{returncdf,morimura2012parametric}. Later work developed both practical approximations \citep{drllinear,mmddrl2} and algorithms \citep{qrdqn-Dabney18,iqn,fqf,dpo,d4pg,mmddrl,splinedrl,sketchdrl} for distributional value learning, together with theoretical analyses \citep{cramer-analysis,drlstat} under several probability metrics (for comparing different probability laws), including Wasserstein distances \citep{wasserstein1,lossdrl} and CDF-based \(L^p\) metrics \citep{lpmetirc}. Among these, the Cram\'er metric is the distinguished \(L^2\) case and is naturally expressed at the level of cumulative distribution functions. Recent work \citep{our1st} has further emphasised the CDF domain itself as the intrinsic analytical setting for Bellman dynamics under the Cram\'er geometry, and established an exact spectral realisation of the corresponding evaluation problem in the non-soft setting.

In maximum-entropy RL, the optimisation objective is augmented by policy entropy, with roots in entropy-regularised control and inference-based formulations \citep{maxentrl2,todorov07,todorov09,toussaint09,rawlik12}. This line of work yields the soft Bellman framework \citep{haarnoja17,neu17,geist19,sac1,sac2}, and underlies modern deep maximum-entropy methods for continuous control \citep{maxentrl1,sac1,sac2}. Existing analyses in this setting, however, are developed mainly at the scalar level of soft value or action-value functions, where the focus is on the expectation of the entropy-regularised return rather than on its full distribution.

These two lines of work meet in distributional formulations of soft policy iteration and related algorithm design \citep{dsacDuan_2022,dsacma,dsactDuan_2025}, where one models entropy-regularised returns rather than only their expectations. This leads to the fixed-policy evaluation problem for the distributional soft Bellman operator. Under the Cram\'er metric, distances between return distributions are \(L^2\) distances between their CDFs. When the distributional soft Bellman update is written at this CDF level, the discount factor acts by rescaling the CDF argument, while the reward and entropy terms enter as a combined random shift. The former gives the \(\sqrt{\gamma}\)-contraction estimate, and the latter determines the condition needed for the update to remain in the finite-Cram\'er-distance domain. Moreover, at this level of analysis, this domain condition can be stated as a first-moment bound on the combined one-step soft shift, rather than separate uniform boundedness assumptions on the reward and entropy terms often used in maximum entropy RL analyses \citep{geist19,sac1,dsacDuan_2022}. The present paper develops this CDF-level and spectral evaluation theory, and thereby identifies the exact Cram\'er-geometric Bellman fixed point associated with the policy evaluation step of DSPI.

\section{Preliminaries and analytical setup}
\label{sec:preliminaries}
This section introduces the probabilistic objects used in soft policy evaluation, specifies the CDF field domain used in the subsequent analysis, and recalls the CDF-to-spectral transport \citep{our1st} needed later.

\subsection{Basic setup for soft policy evaluation}
\label{subsec:basic-notation}
Consider a discounted Markov decision process (MDP) \citep{rlbook,Suttonbook} \(\langle \mathcal S,\mathcal A,\mathcal P,\mathcal R,\gamma\rangle\), where \(\mathcal S\) and \(\mathcal A\) are the state and action spaces, \(\mathcal P(\cdot\mid s,a)\) is the transition kernel on \(\mathcal S\) given \((s,a)\in\mathcal S\times\mathcal A\), \(\mathcal R:\mathcal S\times\mathcal A\times\mathcal S\to\mathbb R\) is the reward function, and \(\gamma\in(0,1)\) is the discount factor. A trajectory is a stochastic process \(\{S_t,A_t,R_t\}_{t\ge0}\) such that $S_{t+1}\sim\mathcal P(\cdot\mid S_t,A_t),$ $R_t=\mathcal R(S_t,A_t,S_{t+1}).$

Throughout the paper, \(\pi\) denotes a fixed stationary policy, where \(\pi(\cdot\mid s)\) is a probability distribution on \(\mathcal A\) for each state \(s\in\mathcal S\). For each \((s,a)\in\mathcal S\times\mathcal A\), let \(s'\) denote the next state sampled from \(\mathcal P(\cdot\mid s,a)\), and then let \(a'\) denote the next action, if $a'\sim \pi(\cdot\mid s')$, we write \((s',a')\sim\mathcal P^\pi(\cdot\mid s,a)\) for this joint sampling scheme. The corresponding one-step reward random variable is denoted by \(R(s,a)\), that is, \(R(s,a)\) has the distribution induced by \(\mathcal R(s,a,s')\) when the next state \(s'\) is generated from \(\mathcal P(\cdot\mid s,a)\).

In the maximum-entropy setting, the return is augmented by discounted entropy terms. For \(t\ge 0\), define the soft return by
\begin{equation}
Z_t^\pi
:=
R_{t} +
\sum_{k=1}^{\infty}\gamma^k [R_{t+k}-\alpha \log \pi(A_{t+k}\mid S_{t+k})],
\label{eq:soft_return_process}
\end{equation}
where \(\alpha>0\) is the entropy-regularisation coefficient. For \((s,a)\in\mathcal S\times\mathcal A\), we then define
\begin{equation}
Z^\pi(s,a)
:=
Z_t^\pi \ \text{conditioned on}\ (S_t,A_t)=(s,a).
\label{eq:soft_return_definition}
\end{equation}
Under the time-homogeneous dynamics induced by the stationary policy \(\pi\), the conditional law above does not depend on the choice of \(t\), so the notation \(Z^\pi(s,a)\) is well defined. By separating the first reward term in \eqref{eq:soft_return_process}, one obtains the Bellman equation
\begin{equation}
Z^\pi(s,a)=
R(s,a)+\gamma\bigl(Z^\pi(s',a')-\alpha\log\pi(a'\mid s')\bigr),
\qquad
(s',a')\sim\mathcal P^\pi(\cdot\mid s,a).
\label{eq:rv_bellman_soft}
\end{equation}

A return-distribution field \citep{c51,drlbook_bellemare_distributional_2023} is a mapping \(\mathcal Z:\mathcal S\times\mathcal A\to\mathscr P(\mathbb R)\), where \(\mathscr P(\mathbb R)\) denotes the set of probability laws on \(\mathbb R\). Thus, for each \((s,a)\in\mathcal S\times\mathcal A\), \(\mathcal Z(s,a)\) is a probability law on \(\mathbb R\). The distributional soft Bellman operator \(\mathcal T_{DS}^\pi\) acts on return-distribution fields by
\begin{equation}
(\mathcal T_{DS}^\pi\mathcal Z)(s,a)
:=
\mathrm{Law}\!\left(
R(s,a)+\gamma\bigl(Z'-\alpha\log\pi(a'\mid s')\bigr)
\right),
\label{eq:dist_soft_bellman_operator}
\end{equation}
where \((s',a')\sim\mathcal P^\pi(\cdot\mid s,a)\) and, conditionally on \((s',a')\), the random variable \(Z'\) has law \(\mathcal Z(s',a')\), i.e. $\mathrm{Law}(Z')=\mathcal Z(s',a')$, equivalently $Z'\sim \mathcal Z(s',a')$. Accordingly, the family of soft return random variables \(\{Z^\pi(s,a)\}_{(s,a)\in\mathcal S\times\mathcal A}\) induces the return-distribution field \(\mathcal Z^\pi:\mathcal S\times\mathcal A\to\mathscr P(\mathbb R)\), where \(\mathcal Z^\pi(s,a):=\mathrm{Law}\bigl(Z^\pi(s,a)\bigr)\). Taking laws on both sides of \eqref{eq:rv_bellman_soft} yields the fixed-point relation \(\mathcal Z^\pi=\mathcal T_{DS}^\pi \mathcal Z^\pi\).

When the first moment is finite, the soft return also induces the corresponding soft action-value function $Q^\pi(s,a):=\mathbb E\bigl[Z^\pi(s,a)\bigr].$ The corresponding expectation-level soft Bellman operator \(\mathcal T^\pi_S\) acts on action-value functions \(Q:\mathcal S\times\mathcal A\to\mathbb R\) by
\begin{equation}
(\mathcal T^\pi_S Q)(s,a)
:=
\mathbb E_{r\sim R(s,a),\,(s',a')\sim\mathcal P^\pi(\cdot\mid s,a)}
\Bigl[
r+\gamma\bigl(Q(s',a')-\alpha\log\pi(a'\mid s')\bigr)
\Bigr].
\label{eq:scalar-soft-bellman-operator}
\end{equation}
By taking expectations in \eqref{eq:rv_bellman_soft}, one likewise obtains the fixed-point relation $Q^\pi=\mathcal T^\pi_S Q^\pi.$

\subsection{CDF field domain}
\label{subsec:cdf-field-domain}
We now pass from return-distribution fields to their CDF representations, since the Cram\'er geometry is expressed directly in terms of CDF. For a probability law \(P\) on \(\mathbb R\), let \(F_P\) denote its CDF \citep{prob_intro}. The Cram\'er distance between probability laws \(P_1\) and \(P_2\) is $d_C(P_1,P_2)
:=
\left(
\int_{\mathbb R}\bigl(F_{P_1}(x)-F_{P_2}(x)\bigr)^2\,dx
\right)^{1/2}.$ Let \(\delta_0\) denote the Dirac law at \(0\). We define the finite-Cram\'er CDF domain by
\begin{equation}
\Gamma_F
:=
\left\{
F_P:
F_P \text{ is a CDF and } d_C(P,\delta_0)<\infty
\right\}.
\label{eq:GammaF_def_softpaper}
\end{equation}
Thus \(\Gamma_F\) consists of those CDFs whose underlying laws have finite Cram\'er distance to \(\delta_0\).

Given a return-distribution field \(\mathcal Z:\mathcal S\times\mathcal A\to\mathscr P(\mathbb R)\) such that \(F_{\mathcal Z(s,a)}\in\Gamma_F\) for every \((s,a)\in\mathcal S\times\mathcal A\), define its associated CDF field pointwise by $F_{\mathcal Z}(s,a)(x):=F_{\mathcal Z(s,a)}(x).$ The admissible CDF field domain is then
\begin{equation}
\mathcal X^{\mathrm{cdf}}
:=
\left\{
F_{\mathcal Z}\ \middle|\ 
F_{\mathcal Z(s,a)}\in\Gamma_F
\text{ for all }(s,a)\in\mathcal S\times\mathcal A,
\ \sup_{(s,a)\in\mathcal S\times\mathcal A}
d_C\bigl(\mathcal Z(s,a),\delta_0\bigr)<\infty
\right\}.
\label{eq:xcdf_def_softpaper}
\end{equation}
Thus \(\mathcal X^{\mathrm{cdf}}\) consists of state-action indexed CDF fields whose pointwise laws lie in the finite-Cram\'er class, with a uniform finite-Cram\'er bound over \(\mathcal S\times\mathcal A\).

For \(F_{\mathcal Z_1},F_{\mathcal Z_2}\in\mathcal X^{\mathrm{cdf}}\), define
\begin{equation}
d_{\mathrm{Cr}}(F_{\mathcal Z_1},F_{\mathcal Z_2})
:=
\sup_{(s,a)\in\mathcal S\times\mathcal A}
\left(
\int_{\mathbb R}
\bigl(F_{\mathcal Z_1}(s,a)(x)-F_{\mathcal Z_2}(s,a)(x)\bigr)^2\,dx
\right)^{1/2}.
\label{eq:cramer_metric_state_action_softpaper}
\end{equation}
This is the state-action supremum of the pointwise Cram\'er distances between the corresponding return laws.

For comparison with the non-soft setting, note that taking \(\alpha=0\) in \eqref{eq:soft_return_process}, \eqref{eq:rv_bellman_soft}, and \eqref{eq:dist_soft_bellman_operator} recovers the standard distributional Bellman recursion. The corresponding CDF-level update is then given by $\mathbb E_{r\sim R(s,a),\,(s',a')\sim\mathcal P^\pi(\cdot\mid s,a)}[F_{\mathcal Z}(s',a')(\frac{x-r}{\gamma})].$ The soft CDF-level operator introduced later is the entropy-regularised counterpart of this update.

\subsection{Spectral transport}
\label{subsec:spectral-transport}
We next introduce the spectral domain objects associated with the CDF field domain. The construction is based on the same finite-Cram\'er setting as above: raw CDFs are taken from \(\Gamma_F\), and CDF fields are taken from \(\mathcal X^{\mathrm{cdf}}\). Fix \(\epsilon>0\). Throughout this subsection, the notation \(\widehat{\cdot}\) denotes the symmetric unitary Fourier transform \citep{fourier,functional1} on \(\mathbb R\), namely $\widehat f(\omega) := \frac{1}{\sqrt{2\pi}} \int_{\mathbb R} f(x)e^{-i\omega x}\,dx.$

Let $F_{\delta_0}(x)=\mathbbm 1\{x\ge 0\}$ denote the CDF of the Dirac law at \(0\), and define the centring map $C(F):=F-F_{\delta_0}.$ Thus the spectral construction acts on centred CDFs rather than on raw CDFs directly. For the fixed \(\epsilon>0\), define 

the centred CDF Hilbert space $\mathcal H^{\mathrm{cdf}}_\epsilon
 :=
 \Bigl\{
 H \mid
 \|H\|^2_{\mathcal H^{\mathrm{cdf}}_\epsilon}
 =
 \int_{\mathbb R}
 \frac{\omega^2}{\omega^2+\epsilon}
 |\widehat H(\omega)|^2\,d\omega
 < \infty
 \Bigr\},$
and the regularised spectral Hilbert space $\mathcal H_{\epsilon}
 :=
 \Bigl\{
 f \mid
 \|f\|^2_{\mathcal H_{\epsilon}}
 =
 \frac{1}{2\pi}
 \int_{\mathbb R}
 \frac{|\widehat f(\omega)|^2}{\omega^2+\epsilon}\,d\omega
 < \infty
 \Bigr\}.$
For \(F\in\Gamma_F\), the centred CDF \(C(F)\) belongs to \(\mathcal H^{\mathrm{cdf}}_\epsilon\), this follows from the corresponding result in the construction in \citep{our1st}. Define the map $\mathcal U:\mathcal H^{\mathrm{cdf}}_\epsilon\to\mathcal H_\epsilon$ through its Fourier transform by
\begin{equation}
\widehat{(\mathcal U H)}(\omega)
:=
-\sqrt{2\pi}\,i\omega\,\overline{\widehat H(\omega)}.
\label{eq:U_formula_softpaper}
\end{equation}
Here \(\overline{\widehat H(\omega)}\) denotes the complex conjugate of \(\widehat H(\omega)\). This convention agrees with the standard Cram\'er spectral construction and sends centred CDFs to their associated spectral objects.

The raw-to-spectral map is then defined by $\mathcal V:=\mathcal U\circ C.$ Hence \(\mathcal V\) is defined on \(\Gamma_F\), maps \(\Gamma_F\) into \(\mathcal H_\epsilon\), and acts on a raw CDF by first centring it and then applying \(\mathcal U\). Its inverse is well defined on its range. Applying \(\mathcal V\) pointwise to CDF fields yields the lifted map
\begin{equation}
\bigl(\mathbb V F_{\mathcal Z}\bigr)(s,a):=\mathcal V\bigl(F_{\mathcal Z}(s,a)\bigr),
\qquad F_{\mathcal Z}\in\mathcal X^{\mathrm{cdf}}.
\label{eq:lifted_transport_softpaper}
\end{equation}

The corresponding spectral field domain is $\mathcal X^{\mathrm{spec}}
:=
\Bigl\{
\nu:\mathcal S\times\mathcal A \to \mathcal H_{\epsilon}
\;\mid\;
\|\nu\|_{\infty,\epsilon}<\infty
\Bigr\}$ for measurable $\nu$, equipped with the supremum norm $
\|\nu\|_{\infty,\epsilon}
:=
\sup_{(s,a)\in\mathcal S\times\mathcal A}
\|\nu(s,a)\|_{\mathcal H_{\epsilon}}.$
We then define the admissible spectral domain by
\begin{equation}
\mathcal X^{\mathrm{adm}}:=\mathrm{range}(\mathbb V)\subset \mathcal X^{\mathrm{spec}}.
\label{eq:xadm_def_softpaper}
\end{equation}
The inverse lifted map is well defined on \(\mathcal X^{\mathrm{adm}}\) by pointwise application of \(\mathcal V^{-1}\). Thus \(\mathbb V\) maps CDF fields in \(\mathcal X^{\mathrm{cdf}}\) into \(\mathcal X^{\mathrm{adm}}\), and \(\mathbb V^{-1}\) recovers the corresponding CDF field on this domain.

\section{CDF-level formulation and soft policy evaluation}
\label{sec:cdf-soft-evaluation}
We now develop the policy evaluation theory of the distributional soft Bellman operator on the CDF field domain. We begin by formulating its action directly at the CDF level, and then use this formulation to establish the main evaluation results under the Cram\'er metric, including well-definedness, contraction, fixed-point existence and uniqueness, and convergence of the associated policy-evaluation iterates.

\subsection{The distributional soft Bellman operator at the CDF level}
\label{subsec:soft-cdf}
We now formulate the distributional soft Bellman operator directly on the CDF field domain, starting from the distribution-level recursion \eqref{eq:dist_soft_bellman_operator}. Since the CDF-level update involves both an affine shift and averaging over successor state-action pairs, we first record the assumptions needed to ensure that the induced operator is well defined on \(\mathcal X^{\mathrm{cdf}}\).

\begin{assumption}[Soft policy admissibility]
\label{ass:soft-policy-admissibility}
For each \(s\in\mathcal S\), if \(a\sim\pi(\cdot\mid s)\), then \(\log\pi(a\mid s)\) is well defined almost surely. Moreover,
\begin{equation}
\label{eq:soft-shift-integrability-assumption}
\sup_{(s,a)\in\mathcal S\times\mathcal A}
\mathbb E_{r\sim R(s,a),\,(s',a')\sim \mathcal P^\pi(\cdot\mid s,a)}
\bigl[
|\,r-\gamma\alpha\log\pi(a'\mid s')\,|
\bigr]
<\infty.
\end{equation}
Finally, for every \(F_{\mathcal Z}\in\mathcal X^{\mathrm{cdf}}\) and \(x\in\mathbb R\), the map
\begin{equation}
\label{eq:soft-measurability-assumption}
(s,a)\mapsto
\mathbb E_{r\sim R(s,a),\,(s',a')\sim \mathcal P^\pi(\cdot\mid s,a)}
\Biggl[
F_{\mathcal Z}(s',a')
\Bigl(
\frac{x-r}{\gamma}
+
\alpha\log\pi(a'\mid s')
\Bigr)
\Biggr]
\end{equation}
is measurable.
\end{assumption}

The central quantitative part of Assumption~\ref{ass:soft-policy-admissibility} is the uniform first-moment condition in \eqref{eq:soft-shift-integrability-assumption}, which controls the one-step soft shift \(r-\gamma\alpha\log\pi(a'\mid s')\).

\begin{remark}[On the strength of Assumption~\ref{ass:soft-policy-admissibility}]
\label{rem:soft-shift-integrability-strength}
The requirement in \eqref{eq:soft-shift-integrability-assumption} controls the first absolute moment of the one-step soft shift \(r-\gamma\alpha\log\pi(a'\mid s')\). In particular, it is implied by stronger boundedness-style conditions used in related entropy-regularised RL analyses, such as uniformly bounded rewards together with uniformly bounded entropy terms, or finite-action assumptions ensuring bounded entropy-augmented rewards \citep{geist19,sac1,dsacDuan_2022}, but does not require either quantity to be almost surely bounded on its own. In this sense, Assumption~\ref{ass:soft-policy-admissibility} isolates the level of regularity actually needed for the CDF-level domain-preservation argument under the Cram\'er geometry.
\end{remark}

For the policy-induced soft return field itself, we additionally assume that its CDF representation lies in the admissible domain.

\begin{assumption}[Admissible soft return CDF field]
\label{assump:soft-xcdf-domain}
For the fixed policy \(\pi\), the policy-induced soft return-distribution field \(\mathcal Z^\pi\) satisfies $F_{\mathcal Z^\pi}\in\mathcal X^{\mathrm{cdf}}.$
\end{assumption}

Under Assumption~\ref{ass:soft-policy-admissibility}, the distribution-level soft Bellman recursion induces the following operator on CDF fields.

\begin{definition}[CDF-level distributional soft Bellman operator]
\label{def:cdf-soft-bellman-operator}
Suppose Assumption~\ref{ass:soft-policy-admissibility} holds. For \(F_{\mathcal Z}\in\mathcal X^{\mathrm{cdf}}\), define
\begin{equation}
\label{eq:soft-bellman-cdf-definition}
(\mathcal T^{\pi}_{\mathrm{cdf}}F_{\mathcal Z})(s,a)(x)
:=
\mathbb E_{r\sim R(s,a),\,(s',a')\sim\mathcal P^\pi(\cdot\mid s,a)}
\Biggl[
F_{\mathcal Z(s',a')}
\Bigl(
\frac{x-r}{\gamma}
+
\alpha\log\pi(a'\mid s')
\Bigr)
\Biggr].
\end{equation}
\end{definition}

Relative to the non-soft CDF-level Bellman update, the only change is the additional shift \(\alpha\log\pi(a'\mid s')\) in the successor argument.

We first verify that this operator preserves the CDF field domain.

\begin{proposition}[The CDF-level soft Bellman operator preserves the CDF field domain]
\label{prop:soft-bellman-self-map}
Suppose Assumption~\ref{ass:soft-policy-admissibility} holds. Then $\mathcal T^{\pi}_{\mathrm{cdf}}:\mathcal X^{\mathrm{cdf}}\to\mathcal X^{\mathrm{cdf}}.$
\end{proposition}

\begin{proof}[Proof sketch]
For each \(F_{\mathcal Z}\in\mathcal X^{\mathrm{cdf}}\), the map in \eqref{eq:soft-bellman-cdf-definition} is obtained by averaging successor CDFs after an affine change of variable with positive slope \(1/\gamma\). It therefore remains a CDF. The uniform finite-Cram\'er bound in \(\mathcal X^{\mathrm{cdf}}\), together with Assumption~\ref{ass:soft-policy-admissibility}, yields preservation of finite Cram\'er distance to \(\delta_0\), while the required measurability is included in the same assumption. Hence \(\mathcal T^{\pi}_{\mathrm{cdf}}F_{\mathcal Z}\in\mathcal X^{\mathrm{cdf}}\). A full proof is given in Appendix~\ref{app:proof-soft-self-map}.
\end{proof}

We next relate \(\mathcal T^{\pi}_{\mathrm{cdf}}\) to the distribution-level soft Bellman operator.

\begin{proposition}[Consistency with the distribution-level soft Bellman operator]
\label{prop:soft-cdf-consistency}
Let \(\mathcal Z\) be a return-distribution field such that \(F_{\mathcal Z}\in\mathcal X^{\mathrm{cdf}}\). Then $F_{\mathcal T^\pi_{DS}\mathcal Z}=\mathcal T^{\pi}_{\mathrm{cdf}}F_{\mathcal Z}.$
\end{proposition}

\begin{proof}[Proof sketch]
Evaluate the CDF of \((\mathcal T^\pi_{DS}\mathcal Z)(s,a)\) at \(x\), condition on \((s',a')\sim\mathcal P^\pi(\cdot\mid s,a)\), and identify the resulting conditional probability with the successor CDF \(F_{\mathcal Z}(s',a')\). This yields \eqref{eq:soft-bellman-cdf-definition}, and hence \(F_{\mathcal T^\pi_{DS}\mathcal Z}=\mathcal T^{\pi}_{\mathrm{cdf}}F_{\mathcal Z}\). A full proof is given in Appendix~\ref{app:proof-soft-cdf-consistency}.
\end{proof}

As an immediate consequence, any distribution-level fixed point induces a corresponding CDF-level fixed point.

\begin{corollary}[Fixed-point correspondence]
\label{cor:soft-fixed-point-correspondence}
Let \(\mathcal Z^\pi\) be a return-distribution field such that \(F_{\mathcal Z^\pi}\in\mathcal X^{\mathrm{cdf}}\) under Assumption~\ref{assump:soft-xcdf-domain}. If $\mathcal T^\pi_{DS}\mathcal Z^\pi=\mathcal Z^\pi,$
then $\mathcal T^{\pi}_{\mathrm{cdf}}F_{\mathcal Z^\pi}=F_{\mathcal Z^\pi}.$
\end{corollary}

\begin{proof}
This follows immediately from Proposition~\ref{prop:soft-cdf-consistency}, since $\mathcal T^{\pi}_{\mathrm{cdf}}F_{\mathcal Z^\pi}=F_{\mathcal T^\pi_{DS}\mathcal Z^\pi}=F_{\mathcal Z^\pi}.$
\end{proof}

This identifies the CDF-level operator as the appropriate formulation of the distribution-level soft Bellman recursion on \(\mathcal X^{\mathrm{cdf}}\). We next turn to the corresponding policy evaluation problem under the Cram\'er metric.

\subsection{Distributional soft policy evaluation}
\label{subsec:soft-evaluation}
Having established the CDF-level formulation of the soft Bellman operator, we now turn to the corresponding soft policy evaluation problem under the Cram\'er metric. Since \(\mathcal T^{\pi}_{\mathrm{cdf}}\) gives the CDF-level representation of the underlying distributional soft Bellman dynamics, the analysis may be carried out directly on \(\mathcal X^{\mathrm{cdf}}\).

The key observation is that the soft Bellman action remains contractive under the Cram\'er metric, with the same factor \(\sqrt{\gamma}\) as in the standard distributional policy evaluation setting.

\begin{theorem}[Cram\'er contraction of the distributional soft Bellman operator]
\label{thm:soft-contraction}
Suppose Assumption~\ref{ass:soft-policy-admissibility} holds. Then \(\mathcal T^{\pi}_{\mathrm{cdf}}\) is a \(\sqrt{\gamma}\)-contraction on \((\mathcal X^{\mathrm{cdf}},d_{\mathrm{Cr}})\). More precisely, for any \(F_{\mathcal Z_1},F_{\mathcal Z_2}\in\mathcal X^{\mathrm{cdf}}\), $d_{\mathrm{Cr}}
\bigl(
\mathcal T^{\pi}_{\mathrm{cdf}}F_{\mathcal Z_1},
\mathcal T^{\pi}_{\mathrm{cdf}}F_{\mathcal Z_2}
\bigr)
\le
\sqrt{\gamma}\,d_{\mathrm{Cr}}(F_{\mathcal Z_1},F_{\mathcal Z_2}).$
\end{theorem}

\begin{proof}[Proof sketch]
For each \((s,a)\), the difference of the two updated CDFs is the expectation of the shifted successor CDF difference. Applying Jensen's inequality \citep{jensen} and the change of variable \(y=(x-r)/\gamma+\alpha\log\pi(a'\mid s')\) gives a factor \(\gamma\) at the level of squared Cram\'er distance. Taking the supremum over \((s,a)\) and then square roots yields $d_{\mathrm{Cr}}
\bigl(
\mathcal T^{\pi}_{\mathrm{cdf}}F_{\mathcal Z_1},
\mathcal T^{\pi}_{\mathrm{cdf}}F_{\mathcal Z_2}
\bigr)
\le
\sqrt{\gamma}\,d_{\mathrm{Cr}}(F_{\mathcal Z_1},F_{\mathcal Z_2})$. A full proof is given in Appendix~\ref{app:proof-soft-contraction}.
\end{proof}

To pass from contraction to soft policy evaluation, we invoke the completeness \citep{completeness} of the ambient CDF field space.

\begin{lemma}[Complete CDF field space \citep{our1st}]
\label{lem:cdf-space-complete}
The metric space \((\mathcal X^{\mathrm{cdf}},d_{\mathrm{Cr}})\) is complete.
\end{lemma}

Combining Theorem~\ref{thm:soft-contraction} with Proposition~\ref{prop:soft-bellman-self-map} and Lemma~\ref{lem:cdf-space-complete} yields the existence and uniqueness of the fixed point for distributional soft evaluation.

\begin{proposition}[Unique fixed point for distributional soft policy evaluation]
\label{prop:soft-fixed-point}
Suppose Assumption~\ref{ass:soft-policy-admissibility} holds. Then \(\mathcal T^{\pi}_{\mathrm{cdf}}\) admits a unique fixed point
$F^\star \in \mathcal X^{\mathrm{cdf}}$ such that $\mathcal T^{\pi}_{\mathrm{cdf}}F^\star = F^\star.$
\end{proposition}

\begin{proof}
By Proposition~\ref{prop:soft-bellman-self-map}, \(\mathcal T^{\pi}_{\mathrm{cdf}}\) is a self-map on \(\mathcal X^{\mathrm{cdf}}\). By Theorem~\ref{thm:soft-contraction}, it is a contraction on \((\mathcal X^{\mathrm{cdf}},d_{\mathrm{Cr}})\). The claim follows from Lemma~\ref{lem:cdf-space-complete} and the Banach fixed-point theorem \citep{banach}.
\end{proof}

Proposition~\ref{prop:soft-fixed-point} should be distinguished from Corollary~\ref{cor:soft-fixed-point-correspondence}. The proposition gives the abstract existence and uniqueness of a fixed point \(F^\star\) for \(\mathcal T^{\pi}_{\mathrm{cdf}}\) on \(\mathcal X^{\mathrm{cdf}}\), whereas the corollary shows that any distribution-level fixed point is represented faithfully at the CDF level. Under Assumption~\ref{assump:soft-xcdf-domain}, the policy-induced field \(F_{\mathcal Z^\pi}\) is such a CDF-level fixed point, and uniqueness therefore yields $F_{\mathcal Z^\pi}=F^\star.$ Thus the abstract fixed point obtained above is precisely the CDF representation of the policy-induced soft return field.

The fixed-point characterisation immediately yields convergence of the associated policy evaluation iterates.

\begin{corollary}[Convergence of soft policy evaluation iterates]
\label{cor:soft-iterate-convergence}
Suppose Assumption~\ref{ass:soft-policy-admissibility}--\ref{assump:soft-xcdf-domain} holds, and let \(F_{\mathcal Z_0}\in\mathcal X^{\mathrm{cdf}}\) be arbitrary. Define the sequence \(\{F_{\mathcal Z_k}\}_{k\ge0}\subset\mathcal X^{\mathrm{cdf}}\) by $F_{\mathcal Z_{k+1}}
=
\mathcal T^{\pi}_{\mathrm{cdf}}F_{\mathcal Z_k}.$
Then $F_{\mathcal Z_k}\to F_{\mathcal Z^\pi}
\text{ in }(\mathcal X^{\mathrm{cdf}},d_{\mathrm{Cr}}),$
where \(F_{\mathcal Z^\pi}\) is the unique fixed point from Proposition~\ref{prop:soft-fixed-point}. More precisely, for every \(k\ge0\), $d_{\mathrm{Cr}}(F_{\mathcal Z_k},F_{\mathcal Z^\pi})\le(\sqrt{\gamma})^k\,d_{\mathrm{Cr}}(F_{\mathcal Z_0},F_{\mathcal Z^\pi}).$
\end{corollary}

\begin{proof}
Applying Theorem~\ref{thm:soft-contraction} to \(F_{\mathcal Z_k}\) and \(F_{\mathcal Z^\pi}\) yields $d_{\mathrm{Cr}}
\bigl(
\mathcal T^{\pi}_{\mathrm{cdf}}F_{\mathcal Z_k},
\mathcal T^{\pi}_{\mathrm{cdf}}F_{\mathcal Z^\pi}
\bigr)
\le
\sqrt{\gamma}\,d_{\mathrm{Cr}}(F_{\mathcal Z_k},F_{\mathcal Z^\pi}).$
Since \(F_{\mathcal Z^\pi}\) is a fixed point and \(F_{\mathcal Z_{k+1}}=\mathcal T^{\pi}_{\mathrm{cdf}}F_{\mathcal Z_k}\), this becomes $d_{\mathrm{Cr}}(F_{\mathcal Z_{k+1}},F_{\mathcal Z^\pi})\le\sqrt{\gamma}\,d_{\mathrm{Cr}}(F_{\mathcal Z_k},F_{\mathcal Z^\pi}).$
Iterating the inequality proves $d_{\mathrm{Cr}}(F_{\mathcal Z_k},F_{\mathcal Z^\pi})\le(\sqrt{\gamma})^k\,d_{\mathrm{Cr}}(F_{\mathcal Z_0},F_{\mathcal Z^\pi})$, hence \(F_{\mathcal Z_k}\to F_{\mathcal Z^\pi}\).
\end{proof}

These results show that distributional soft policy evaluation is well posed on the admissible CDF field domain. In particular, the CDF-level distributional soft Bellman operator admits a unique fixed point and generates convergent policy evaluation iterates under the Cram\'er metric. This completes the distributional soft policy evaluation theory at the CDF level.

\section{Spectral form of the distributional soft Bellman dynamics}
\label{sec:spectral-soft}
We now transfer the CDF-level formulation to the spectral domain. Since \(\mathbb V\) is already defined on the admissible CDF field domain in~\eqref{eq:lifted_transport_softpaper}, the spectral soft Bellman operator is obtained from \(\mathcal T^{\pi}_{\mathrm{cdf}}\) by conjugation \citep{operator-conjugate}. The resulting spectral formulation therefore represents exactly the same soft policy evaluation dynamics as at the CDF level.

\begin{definition}[Spectral distributional soft Bellman operator]
\label{def:spectral-soft-bellman}
Define the spectral distributional soft Bellman operator on \(\mathcal X^{\mathrm{adm}}\) by $\mathcal T^{\pi}:=\mathbb V\,\mathcal T^{\pi}_{\mathrm{cdf}}\,\mathbb V^{-1}.$
\end{definition}

Since \(\mathcal X^{\mathrm{adm}}=\mathrm{range}(\mathbb V)\) as in~\eqref{eq:xadm_def_softpaper}, the map \(\mathbb V:\mathcal X^{\mathrm{cdf}}\to\mathcal X^{\mathrm{adm}}\) is bijective. Together with Proposition~\ref{prop:soft-bellman-self-map}, this shows that $\mathcal T^{\pi}:\mathcal X^{\mathrm{adm}}\to\mathcal X^{\mathrm{adm}}$ is a well-defined self-map.

The same conjugation also transports the Cram\'er geometry from the CDF field domain to the spectral domain. In particular, \(\mathcal X^{\mathrm{adm}}\) may be equipped with the induced metric
\begin{equation}
\label{eq:spectral-induced-metric}
\widetilde d_{\mathrm{Cr}}(\nu_1,\nu_2)
:=
d_{\mathrm{Cr}}(\mathbb V^{-1}\nu_1,\mathbb V^{-1}\nu_2),
\qquad
\nu_1,\nu_2\in\mathcal X^{\mathrm{adm}}.
\end{equation}
Thus the spectral formulation represents not only the same evaluation dynamics, but also the same evaluation geometry under transport by \(\mathbb V\). The first point to verify is that the spectral soft Bellman dynamics are exactly the transported image of the CDF-level dynamics.

\begin{proposition}[Transported soft Bellman dynamics]
\label{prop:spectral-soft-decomposition}
For every \(F_{\mathcal Z}\in\mathcal X^{\mathrm{cdf}}\),
\begin{equation}
\label{eq:spectral-soft-decomposition}
\mathbb V\bigl(\mathcal T^{\pi}_{\mathrm{cdf}}F_{\mathcal Z}\bigr)
=
\mathcal T^{\pi}\bigl(\mathbb V F_{\mathcal Z}\bigr).
\end{equation}
\end{proposition}

\begin{proof}
Let \(F_{\mathcal Z}\in\mathcal X^{\mathrm{cdf}}\). By Definition~\ref{def:spectral-soft-bellman}, $\mathcal T^{\pi}\bigl(\mathbb V F_{\mathcal Z}\bigr)=\mathbb V\,\mathcal T^{\pi}_{\mathrm{cdf}}\,\mathbb V^{-1}\bigl(\mathbb V F_{\mathcal Z}\bigr)=\mathbb V\bigl(\mathcal T^{\pi}_{\mathrm{cdf}}F_{\mathcal Z}\bigr),$ which proves \eqref{eq:spectral-soft-decomposition}.
\end{proof}

The above proposition immediately induce the corresponding spectral fixed-point relation.

\begin{corollary}[Spectral fixed-point correspondence]
\label{cor:spectral-soft-fixed-point}
Suppose Assumption~\ref{ass:soft-policy-admissibility}--\ref{assump:soft-xcdf-domain} holds, and let \(F_{\mathcal Z^\pi}\in\mathcal X^{\mathrm{cdf}}\) be the unique fixed point of \(\mathcal T^{\pi}_{\mathrm{cdf}}\) from Proposition~\ref{prop:soft-fixed-point}. Then $\nu^\pi:=\mathbb V F_{\mathcal Z^\pi}\in \mathcal X^{\mathrm{adm}}$ is the unique fixed point of \(\mathcal T^{\pi}\) on \(\mathcal X^{\mathrm{adm}}\), that is, $\mathcal T^{\pi}\nu^\pi=\nu^\pi.$
\end{corollary}

\begin{proof}
Since \(F_{\mathcal Z^\pi}\) satisfies $\mathcal T^{\pi}_{\mathrm{cdf}}F_{\mathcal Z^\pi}=F_{\mathcal Z^\pi},$ applying \(\mathbb V\) to both sides and using Definition~\ref{def:spectral-soft-bellman} gives $\mathcal T^{\pi}(\mathbb V F_{\mathcal Z^\pi})=\mathbb V(\mathcal T^{\pi}_{\mathrm{cdf}}F_{\mathcal Z^\pi})=\mathbb V F_{\mathcal Z^\pi}.$ Thus \(\nu^\pi=\mathbb V F_{\mathcal Z^\pi}\) is a fixed point of \(\mathcal T^{\pi}\). 

If \(\upsilon\in\mathcal X^{\mathrm{adm}}\) is another fixed point, then \(\mathbb V^{-1}\upsilon\in\mathcal X^{\mathrm{cdf}}\) satisfies $\mathcal T^{\pi}_{\mathrm{cdf}}(\mathbb V^{-1}\upsilon)=\mathbb V^{-1}\upsilon$ by Definition~\ref{def:spectral-soft-bellman}. By uniqueness of the CDF-level fixed point from Proposition~\ref{prop:soft-fixed-point}, $\mathbb V^{-1}\upsilon = F_{\mathcal Z^\pi}.$
Applying \(\mathbb V\) again yields \(\upsilon=\nu^\pi\).
\end{proof}

The same transport mechanism also identifies the spectral iterates with the image of the CDF-level policy evaluation iterates.

\begin{corollary}[Spectral iterate correspondence]
\label{cor:spectral-soft-iterate-correspondence}
Suppose Assumption~\ref{ass:soft-policy-admissibility}--\ref{assump:soft-xcdf-domain} holds. Let \(F_{\mathcal Z_0}\in\mathcal X^{\mathrm{cdf}}\), and define $F_{\mathcal Z_{k+1}}=\mathcal T^{\pi}_{\mathrm{cdf}}F_{\mathcal Z_k}.$ For each \(k\ge0\), define $\nu_k:=\mathbb V F_{\mathcal Z_k}.$ Then $\nu_{k+1}=\mathcal T^{\pi}\nu_k$.
Moreover, if \(F_{\mathcal Z^\pi}\) is the unique fixed point of \(\mathcal T^{\pi}_{\mathrm{cdf}}\) and \(\nu^\pi=\mathbb V F_{\mathcal Z^\pi}\), then $\nu_k \to \nu^\pi$ by transport of the CDF-level convergence from Corollary~\ref{cor:soft-iterate-convergence}.
\end{corollary}

\begin{proof}
By Proposition~\ref{prop:spectral-soft-decomposition}, $\nu_{k+1}=\mathbb V F_{\mathcal Z_{k+1}}=\mathbb V(\mathcal T^{\pi}_{\mathrm{cdf}}F_{\mathcal Z_k})=\mathcal T^{\pi}(\mathbb V F_{\mathcal Z_k})=\mathcal T^{\pi}\nu_k,$ which forms the spectral domain iteration. The convergence statement follows by transporting the CDF-level iterates from Corollary~\ref{cor:soft-iterate-convergence}.
\end{proof}

These results provide a spectral realisation of distributional soft evaluation. The spectral domain is not a separate approximation model, but the transported image of the CDF-level soft Bellman dynamics under \(\mathbb V\). In particular, the spectral fixed point and spectral iterates are exactly the images of their CDF-level counterparts under the same transport. Additional direct results of this transport on the spectral domain, including the induced-metric contraction viewpoint, are recorded in Appendix~\ref{app:spectral-consequences}.

\section{Relation to distributional soft policy iteration}
\label{sec:dspi-perspective}
We now explain how the evaluation theory developed in Sections~\ref{sec:cdf-soft-evaluation} and~\ref{sec:spectral-soft} enters the framework of distributional soft policy iteration. The key point is that the soft return process introduced in Section~\ref{subsec:basic-notation} provides the state-action conditioned evaluation object underlying this iteration.

Let \(\rho\) be an initial-state distribution. Under a stationary policy \(\pi\), the maximum-entropy objective \citep{haarnoja17,sac1,sac2,neu17,geist19} is
\begin{equation}
\label{eq:maxent-objective}
J(\pi;\rho)
:=
\mathbb E_{
\substack{
S_t\sim\rho,\\
A_{t+k}\sim\pi(\cdot\mid S_{t+k}),\,k\ge0,\\
S_{t+k+1}\sim\mathcal P(\cdot\mid S_{t+k},A_{t+k}),\\
R_{t+k}=\mathcal R(S_{t+k},A_{t+k},S_{t+k+1})
}
}
\!\left[
\sum_{k=0}^{\infty}\gamma^k\bigl(R_{t+k}-\alpha\log\pi(A_{t+k}\mid S_{t+k})\bigr)
\right].
\end{equation}
This is the trajectory-level entropy-regularised performance criterion, standard in maximum-entropy reinforcement learning, in which the entire trajectory, including the initial action at time \(t\), is generated by \(\pi\). By contrast, the soft return process \(Z_t^\pi\) in~\eqref{eq:soft_return_process} is the state-action conditioned evaluation object used for soft action-value analysis after conditioning on an initial state-action pair \((S_t,A_t)=(s,a)\). In that evaluation object, the initial action is already prescribed rather than sampled from \(\pi\), so the entropy regularisation begins only with subsequent policy-generated actions. Conditioning this process on \((S_t,A_t)=(s,a)\) yields the return variable \(Z^\pi(s,a)\), whose expectation defines the soft action-value function relevant for policy improvement at a prescribed initial action. When the first moment is finite, the corresponding soft action-value function is given by \(Q^\pi(s,a):=\mathbb E\bigl[Z^\pi(s,a)\bigr],\) and the associated soft state-value function is \(V^\pi(s):=\mathbb E_{a\sim\pi(\cdot\mid s)}\bigl[Q^\pi(s,a)-\alpha\log\pi(a\mid s)\bigr].\) Accordingly,
\begin{equation}
\label{eq:maxent-objective-v-form}
J(\pi;\rho)=\mathbb E_{S_t\sim\rho}[V^\pi(S_t)].
\end{equation}
Thus \(J(\pi;\rho)\) provides the trajectory-level performance criterion, while \(V^\pi\) gives its state-conditioned form. Since the policy-improvement step in soft policy iteration is defined statewise \citep{neu17,geist19,sac1,sac2}, the latter will be the more convenient form below.

The present paper develops this distributional soft evaluation theory explicitly. For each fixed policy \(\pi\), the conditional law of \(Z_t^\pi\) given \((S_t,A_t)=(s,a)\) is the distributional object \(\mathcal Z^\pi(s,a)\) introduced in Section~\ref{subsec:basic-notation}. The evaluation theory developed in Sections~\ref{sec:cdf-soft-evaluation} and~\ref{sec:spectral-soft} shows that this evaluation object admits equivalent representations at the distributional, CDF, and spectral levels, namely \(\mathcal Z^\pi\), \(F_{\mathcal Z^\pi}\), and \(\nu^\pi:=\mathbb V F_{\mathcal Z^\pi}\). Here \(\mathcal Z^\pi\) is the fixed point of the distribution-level soft Bellman operator \(\mathcal T^\pi_{DS}\), \(F_{\mathcal Z^\pi}\in\mathcal X^{\mathrm{cdf}}\) is the unique fixed point of the CDF-level operator \(\mathcal T_{\mathrm{cdf}}^{\pi,\alpha}\), and \(\nu^\pi\in\mathcal X^{\mathrm{adm}}\) is the corresponding fixed point of the spectral operator \(\mathcal T^{\pi}\). When the first moment is finite, these exact evaluation objects induce the same soft critic \(Q^\pi\), and hence the same soft state-value function \(V^\pi\). Recall from \eqref{eq:scalar-soft-bellman-operator} that the scalar soft Bellman operator is
\[
(\mathcal T^\pi_S Q)(s,a)
:=
\mathbb E_{r\sim R(s,a),\,(s',a')\sim\mathcal P^\pi(\cdot\mid s,a)}
\Bigl[
r+\gamma\bigl(Q(s',a')-\alpha\log\pi(a'\mid s')\bigr)
\Bigr].
\]
The following proposition shows that the distributional fixed point induces the corresponding expectation-level one.

\begin{proposition}[Expectation-level fixed-point consistency]
\label{prop:soft-expectation-consistency}
Suppose Assumption~\ref{ass:soft-policy-admissibility}--\ref{assump:soft-xcdf-domain} holds, and let \(\mathcal Z^\pi\) be the fixed point of the distribution-level soft Bellman operator \(\mathcal T^\pi_{DS}\). Assume that the first moment of \(Z^\pi(s,a)\) is finite for every \((s,a)\in\mathcal S\times\mathcal A\). Then the induced soft action-value function \(Q^\pi\) satisfies \(Q^\pi=\mathcal T^\pi_S Q^\pi.\)
\end{proposition}

\begin{proof}
Since \(\mathcal Z^\pi\) is the fixed point of \(\mathcal T^\pi_{DS}\), the corresponding return random variable satisfies \(Z^\pi(s,a)
\overset{D}{=}
R(s,a)+\gamma\bigl(Z^\pi(s',a')-\alpha\log\pi(a'\mid s')\bigr),\) where \((s',a')\sim\mathcal P^\pi(\cdot\mid s,a).\)
By the assumed finiteness of the first moments, taking expectations on both sides is legitimate and yields
\[
Q^\pi(s,a)
=
\mathbb E_{r\sim R(s,a),\,(s',a')\sim\mathcal P^\pi(\cdot\mid s,a)}
\Bigl[
r+\gamma\bigl(Q^\pi(s',a')-\alpha\log\pi(a'\mid s')\bigr)
\Bigr].
\]
By \eqref{eq:scalar-soft-bellman-operator}, this is exactly \(Q^\pi(s,a)=(\mathcal T^\pi_S Q^\pi)(s,a).\) Hence \(Q^\pi=\mathcal T^\pi_S Q^\pi.\)
\end{proof}

Consider the \(k\)-th policy-iteration step, with \(k\ge0\). Once the soft return under the current policy \(\pi_k\) has been evaluated and the induced critic \(Q^{\pi_k}\) is available, the policy-improvement step takes the standard maximum-entropy form. Its relation to \eqref{eq:maxent-objective} is mediated by the state-value representation \(V^\pi\): the scalar objective is the expectation of \(V^\pi\) under the initial-state distribution, while the actual improvement step is defined pointwise in the state variable.

\begin{lemma}[Soft policy improvement \citep{haarnoja17,sac1,sac2}]
\label{lem:soft-policy-improvement}
Let \(\pi_k\) be a policy, and let \(Q^{\pi_k}\) be its soft action-value function. Define \(\pi_{k+1}\) statewise by
\begin{equation}
\label{eq:soft-policy-improvement-step}
\pi_{k+1}(\cdot\mid s)
\in
\arg\max_{\pi(\cdot\mid s)}
\mathbb E_{a\sim\pi(\cdot\mid s)}
\bigl[
Q^{\pi_k}(s,a)-\alpha\log\pi(a\mid s)
\bigr],
\qquad s\in\mathcal S.
\end{equation}
Assume that the corresponding soft action-value function \(Q^{\pi_{k+1}}\) exists. Then
\begin{equation}
\label{eq:soft-policy-improvement-ineq}
Q^{\pi_{k+1}}(s,a)\ge Q^{\pi_k}(s,a),
\qquad
\forall (s,a)\in\mathcal S\times\mathcal A.
\end{equation}
\end{lemma}

Combining the fixed-policy evaluation established above with Lemma~\ref{lem:soft-policy-improvement}, we obtain the corresponding distributional soft policy-iteration scheme, in the sense that the present theory provides the exact evaluation component underlying DSPI-style formulations \citep{dsacDuan_2022,dsacma,dsactDuan_2025}.

\begin{proposition}[Distributional soft policy evaluation induces DSPI]
\label{prop:exact-evaluation-dspi}
Suppose that, for each \(k\ge0\), the policy \(\pi_k\) satisfies the admissibility conditions required for the fixed-policy theory developed in Sections~\ref{sec:cdf-soft-evaluation} and~\ref{sec:spectral-soft}, so that the associated evaluation objects \(\mathcal Z^{\pi_k}\), \(F_{\mathcal Z^{\pi_k}}\), \(\nu^{\pi_k}\), and \(Q^{\pi_k}(s,a)=\mathbb E\bigl[Z^{\pi_k}(s,a)\bigr]\) are well defined. If \(\pi_{k+1}\) is defined from \(Q^{\pi_k}\) by \eqref{eq:soft-policy-improvement-step}, then the iteration \(\pi_k\longmapsto\bigl((\mathcal Z^{\pi_k},F_{\mathcal Z^{\pi_k}},\nu^{\pi_k}),\,Q^{\pi_k}\bigr)\longmapsto\pi_{k+1}\) forms a distributional soft policy-iteration scheme with evaluation under the Cram\'er geometry. Moreover, \(Q^{\pi_{k+1}}(s,a)\ge Q^{\pi_k}(s,a),\) \(\forall (s,a)\in\mathcal S\times\mathcal A,\) and
\[
V^{\pi_{k+1}}(s)\ge V^{\pi_k}(s),
\qquad
\forall s\in\mathcal S.
\]
It follows that, for any initial-state distribution \(\rho\), \(J(\pi_{k+1};\rho)\ge J(\pi_k;\rho).\)
\end{proposition}

\begin{proof}[Proof sketch]
The fixed-policy evaluation objects \(\mathcal Z^{\pi_k}\), \(F_{\mathcal Z^{\pi_k}}\), \(\nu^{\pi_k}\), and \(Q^{\pi_k}\) are provided by the theory developed in Sections~\ref{sec:cdf-soft-evaluation} and~\ref{sec:spectral-soft}. The \(Q\)-level improvement \(Q^{\pi_{k+1}}(s,a)\ge Q^{\pi_k}(s,a)\) for all \((s,a)\in\mathcal S\times\mathcal A\) is Lemma~\ref{lem:soft-policy-improvement}. The state value improvement \(V^{\pi_{k+1}}(s)\ge V^{\pi_k}(s)\) then follows from the defining optimality of \(\pi_{k+1}(\cdot\mid s)\) in \eqref{eq:soft-policy-improvement-step} with the \(Q\)-level inequality. The objective improvement \(J(\pi_{k+1};\rho)\ge J(\pi_k;\rho)\) follows from \eqref{eq:maxent-objective-v-form}. A full proof is given in Appendix~\ref{appendix:dspi-proof}.
\end{proof}

The preceding discussion places the evaluation theory developed in Sections~\ref{sec:cdf-soft-evaluation} and~\ref{sec:spectral-soft} within distributional soft policy iteration. For each current policy \(\pi_k\), these results provide the corresponding evaluation objects at the distributional, CDF, and spectral levels, together with the induced soft critic \(Q^{\pi_k}\) and state-value function \(V^{\pi_k}\). Combined with the standard maximum-entropy improvement step of Lemma~\ref{lem:soft-policy-improvement}, Proposition~\ref{prop:exact-evaluation-dspi} identifies the present contribution as the exact evaluation side of DSPI under the Cram\'er geometry.

This also clarifies the scope of the present paper. The analysis developed here concerns the unprojected evaluation component under the Cram\'er geometry. Questions involving projected or restricted critics, approximation error, and its propagation through repeated policy improvement are deferred to Section~\ref{sec:discussion}.

\section{Discussion}
\label{sec:discussion}
The preceding analysis establishes the exact policy evaluation theory of the distributional soft Bellman operator under the Cram\'er geometry. More precisely, for each fixed policy, the soft evaluation problem can be formulated at the distributional level, at the CDF level, and after spectral conjugation. The CDF level is where the Cram\'er metric is represented as an \(L^2\) distance between CDFs, while the spectral level gives an equivalent Hilbert-space representation of the same CDF-level dynamics. From this viewpoint, the present paper identifies the exact evaluation object that underlies Cram\'er-based distributional soft policy iteration.

This exactness also clarifies the scope of the present work. The results developed here concern the unprojected evaluation problem, where the Bellman target is represented without approximation and the analysis is carried out directly under the Cram\'er geometry, in which the CDF-level operator is contractive. The only role of the first-moment assumption is to ensure domain preservation: after applying the soft Bellman update, the resulting CDF field remains in the admissible CDF domain on which the Cram\'er distance is finite.\citep{geist19,sac1,dsacDuan_2022}.

A natural next step is therefore to study restricted or approximate critics in DSPI- and DSAC-style methods \citep{dsacDuan_2022,dsactDuan_2025}. From the perspective of algorithm design, the exact fixed point characterised in this paper may be interpreted as an oracle (distributional/CDF/spectral level) Bellman target for exact soft evaluation under the Cram\'er geometry. This oracle-target viewpoint provides a precise reference relative to which approximation questions can be formulated: whether a restricted critic remains close to the exact Bellman target, whether the induced evaluation error can be controlled, and whether one-step Bellman error bounds can be established. Moreover, because exact fixed-policy evaluation is contractive under the Cram\'er geometry, these questions at least arise within a well-defined Bellman evaluation framework. This is the sense in which the Cram\'er analysis is useful for approximate critics: it supplies a contractive reference problem before one introduces a particular critic class or distribution-matching loss. Objectives such as KL-based distribution matching may still be effective empirically, but their use as critic losses does not by itself give a Bellman-level contraction guarantee \citep{c51,drlbook_bellemare_distributional_2023,cramer-analysis}. In this sense, the present theory supplies a clearer analytical starting point for studying approximate critics. 

A further issue arises in practical actor-critic methods when the target itself is no longer treated as exact, but is instead affected by the current critic approximation. In that case, the problem is no longer only one of approximating a fixed oracle Bellman target, but also of understanding the coupled error generated by approximate evaluation and repeated policy improvement. The present paper does not address these control-level questions, but it provides the exact evaluation framework relative to which they can be formulated precisely.

More broadly, although the present analysis is developed under the Cram\'er geometry, the underlying viewpoint is not confined to this metric alone. One may similarly ask, under other distributional geometries, what the corresponding exact evaluation structure is, whether an intrinsic analytical level analogous to the CDF representation exists, and what class of Bellman-compatible critic parameterisations such a geometry would naturally suggest. In this sense, the contribution of the present paper is both specific and methodological: specific to Cram\'er-based distributional soft evaluation, and methodological in showing how the choice of distributional geometry can determine a representation in which the Bellman update and the metric used to analyse it are expressed directly.

\section{Conclusion}
\label{sec:conclusion}
In this paper, we established the policy evaluation theory of the distributional soft Bellman operator under the Cram\'er geometry. Working on an admissible CDF field domain, we formulated the induced CDF-level recursion and showed that the resulting operator is well defined and remains a \(\sqrt{\gamma}\)-contraction. Consequently, the associated distributional soft evaluation problem admits a unique fixed point together with convergent policy evaluation iterations. The first-moment condition on the one-step soft shift is used to ensure that the soft Bellman update maps the admissible CDF field domain into itself, so that the updated return distributions remain within the class on which the Cram\'er distance is finite \citep{geist19,sac1,dsacDuan_2022}. We then showed that the same evaluation problem admits an exact spectral realisation by conjugation. In this way, the distributional, CDF-level, and spectral formulations are linked as equivalent representations of the same soft Bellman evaluation problem. The role of the CDF formulation is to make explicit when the entropy-regularised Bellman target remains inside the finite-Cram\'er-distance domain, so that the \(\sqrt{\gamma}\)-contraction and fixed-point guarantees apply. This gives stable reference target before introducing critics or empirical distribution matching losses. In particular, the required condition is a first-moment bound on the combined one-step reward-entropy shift, rather than separate boundedness assumptions on the reward and entropy terms. Taken together, these results furnish the evaluation component of Cram\'er-based distributional soft policy iteration with a precise theoretical foundation. They also provide a rigorous reference point for later algorithmic developments under the same geometry, including restricted-critic analysis, Bellman-compatible critic losses, and critic parameterisations formulated at the distributional, CDF, or spectral level.

\acks{This publication has emanated from research supported in part by a grant from Taighde Éireann - Research Ireland under Grant number 18/CRT/6049. For the purpose of Open Access, the author has applied a CC BY public copyright licence to any Author Accepted Manuscript version arising from this submission.}

\bibliography{references}

\appendix

\section{Proofs details for the main text}
\label{app:omitted-proofs}

\subsection{Proof of Proposition~\ref{prop:soft-bellman-self-map}}
\label{app:proof-soft-self-map}
\begin{proof}
Let \(F_{\mathcal Z}\in\mathcal X^{\mathrm{cdf}}\). Fix \((s,a)\in\mathcal S\times\mathcal A\). For each \(r\in\mathbb R\) and \((s',a')\in\mathcal S\times\mathcal A\), consider the function
\[
x\mapsto
F_{\mathcal Z}(s',a')
\Bigl(
\frac{x-r}{\gamma}
+
\alpha\log\pi(a'\mid s')
\Bigr).
\]
Because \(F_{\mathcal Z}\in\mathcal X^{\mathrm{cdf}}\), by definition we have
\(F_{\mathcal Z}(s',a')\in\Gamma_F\) for every \((s',a')\in\mathcal S\times\mathcal A\).
Hence the map \(F_{\mathcal Z}(s',a'):\mathbb R\to[0,1]\) is non-decreasing, right-continuous, and satisfies
\[
\lim_{y\to-\infty}F_{\mathcal Z}(s',a')(y)=0,
\qquad
\lim_{y\to+\infty}F_{\mathcal Z}(s',a')(y)=1.
\]
Moreover, the affine map
\[
x\mapsto \frac{x-r}{\gamma}+\alpha\log\pi(a'\mid s')
\]
has positive slope \(1/\gamma>0\). Hence the composition
\[
x\mapsto
F_{\mathcal Z}(s',a')
\Bigl(
\frac{x-r}{\gamma}
+
\alpha\log\pi(a'\mid s')
\Bigr)
\]
is again non-decreasing and right-continuous, with limits
\[
\lim_{x\to-\infty}
F_{\mathcal Z}(s',a')
\Bigl(
\frac{x-r}{\gamma}
+
\alpha\log\pi(a'\mid s')
\Bigr)
=0,
\]
and
\[
\lim_{x\to+\infty}
F_{\mathcal Z}(s',a')
\Bigl(
\frac{x-r}{\gamma}
+
\alpha\log\pi(a'\mid s')
\Bigr)
=1.
\]

We next verify that the expectation in \eqref{eq:soft-bellman-cdf-definition} is also a CDF. Let \(x_1\le x_2\). Since the quantity inside the expectation is non-decreasing in \(x\), we have
\begin{align*}
&
F_{\mathcal Z}(s',a')
\Bigl(
\frac{x_1-r}{\gamma}
+
\alpha\log\pi(a'\mid s')
\Bigr)
\\
&\hspace{3cm}\le
F_{\mathcal Z}(s',a')
\Bigl(
\frac{x_2-r}{\gamma}
+
\alpha\log\pi(a'\mid s')
\Bigr)
\end{align*}
for all \(r\) and \((s',a')\). Taking expectation with respect to \(r\sim R(s,a)\) and \((s',a')\sim\mathcal P^\pi(\cdot\mid s,a)\) yields
\[
(\mathcal T^{\pi}_{\mathrm{cdf}}F_{\mathcal Z})(s,a)(x_1)
\le
(\mathcal T^{\pi}_{\mathrm{cdf}}F_{\mathcal Z})(s,a)(x_2),
\]
so \(x\mapsto (\mathcal T^{\pi}_{\mathrm{cdf}}F_{\mathcal Z})(s,a)(x)\) is non-decreasing.

Next, let \(x_n\downarrow x\). For each \(r\) and \((s',a')\), right-continuity of \(F_{\mathcal Z}(s',a')\) gives
\[
F_{\mathcal Z}(s',a')
\Bigl(
\frac{x_n-r}{\gamma}
+
\alpha\log\pi(a'\mid s')
\Bigr)
\to
F_{\mathcal Z}(s',a')
\Bigl(
\frac{x-r}{\gamma}
+
\alpha\log\pi(a'\mid s')
\Bigr).
\]
Since every CDF takes values in \([0,1]\), the quantity inside the expectation is bounded by \(1\). Dominated convergence therefore implies
\[
(\mathcal T^{\pi}_{\mathrm{cdf}}F_{\mathcal Z})(s,a)(x_n)
\to
(\mathcal T^{\pi}_{\mathrm{cdf}}F_{\mathcal Z})(s,a)(x),
\]
so \(x\mapsto (\mathcal T^{\pi}_{\mathrm{cdf}}F_{\mathcal Z})(s,a)(x)\) is right-continuous.

Finally, for each \(r\) and \((s',a')\),
\[
F_{\mathcal Z}(s',a')
\Bigl(
\frac{x-r}{\gamma}
+
\alpha\log\pi(a'\mid s')
\Bigr)
\to 0
\qquad\text{as }x\to-\infty,
\]
and
\[
F_{\mathcal Z}(s',a')
\Bigl(
\frac{x-r}{\gamma}
+
\alpha\log\pi(a'\mid s')
\Bigr)
\to 1
\qquad\text{as }x\to+\infty.
\]
Again using dominated convergence, we obtain
\[
\lim_{x\to-\infty}
(\mathcal T^{\pi}_{\mathrm{cdf}}F_{\mathcal Z})(s,a)(x)=0,
\qquad
\lim_{x\to+\infty}
(\mathcal T^{\pi}_{\mathrm{cdf}}F_{\mathcal Z})(s,a)(x)=1.
\]
Hence \(x\mapsto (\mathcal T^{\pi}_{\mathrm{cdf}}F_{\mathcal Z})(s,a)(x)\) is a CDF.

It remains to verify that it belongs to \(\Gamma_F\). For each \(r\) and \((s',a')\), the change of variable
\[
y=\frac{x-r}{\gamma}+\alpha\log\pi(a'\mid s')
\]
gives
\begin{align*}
&
\int_{\mathbb R}
\Biggl(
F_{\mathcal Z}(s',a')
\Bigl(
\frac{x-r}{\gamma}
+
\alpha\log\pi(a'\mid s')
\Bigr)
-
\mathbbm 1\{x\ge 0\}
\Biggr)^2
\,dx
\\
&\le
2\int_{\mathbb R}
\Biggl(
F_{\mathcal Z}(s',a')
\Bigl(
\frac{x-r}{\gamma}
+
\alpha\log\pi(a'\mid s')
\Bigr)
-
\mathbbm 1
\Bigl\{
\frac{x-r}{\gamma}
+
\alpha\log\pi(a'\mid s')
\ge 0
\Bigr\}
\Biggr)^2
\,dx
\\
&\qquad
+
2\int_{\mathbb R}
\Biggl(
\mathbbm 1
\Bigl\{
\frac{x-r}{\gamma}
+
\alpha\log\pi(a'\mid s')
\ge 0
\Bigr\}
-
\mathbbm 1\{x\ge 0\}
\Biggr)^2
\,dx
\\
&=
2\gamma
\int_{\mathbb R}
\bigl(
F_{\mathcal Z}(s',a')(y)-\mathbbm 1\{y\ge0\}
\bigr)^2\,dy
+
2\bigl|\,r-\gamma\alpha\log\pi(a'\mid s')\,\bigr|.
\end{align*}
Since \(F_{\mathcal Z}\in\mathcal X^{\mathrm{cdf}}\), we have
\[
\sup_{(u,v)\in\mathcal S\times\mathcal A}
d_C\bigl(\mathcal Z(u,v),\delta_0\bigr)<\infty.
\]
Hence, taking expectation with respect to \(r\sim R(s,a)\) and \((s',a')\sim\mathcal P^\pi(\cdot\mid s,a)\), we obtain
\begin{align*}
&
\int_{\mathbb R}
\Bigl(
(\mathcal T^{\pi}_{\mathrm{cdf}}F_{\mathcal Z})(s,a)(x)-\mathbbm 1\{x\ge0\}
\Bigr)^2
\,dx
\\
&=
\int_{\mathbb R}
\Biggl(
\mathbb E_{r\sim R(s,a),\,(s',a')\sim\mathcal P^\pi(\cdot\mid s,a)}
\Biggl[
F_{\mathcal Z}(s',a')
\Bigl(
\frac{x-r}{\gamma}
+
\alpha\log\pi(a'\mid s')
\Bigr)
\Biggr]
-
\mathbbm 1\{x\ge0\}
\Biggr)^2
\,dx
\\
&\le
\mathbb E_{r\sim R(s,a),\,(s',a')\sim\mathcal P^\pi(\cdot\mid s,a)}
\Biggl[
\int_{\mathbb R}
\Biggl(
F_{\mathcal Z}(s',a')
\Bigl(
\frac{x-r}{\gamma}
+
\alpha\log\pi(a'\mid s')
\Bigr)
-
\mathbbm 1\{x\ge0\}
\Biggr)^2
\,dx
\Biggr]
\\
&\le
2\gamma
\sup_{(u,v)\in\mathcal S\times\mathcal A}
d_C^2\bigl(\mathcal Z(u,v),\delta_0\bigr)
+
2\,
\mathbb E_{r\sim R(s,a),\,(s',a')\sim\mathcal P^\pi(\cdot\mid s,a)}
\bigl[
|\,r-\gamma\alpha\log\pi(a'\mid s')\,|
\bigr].
\end{align*}
By the definition of \(\mathcal X^{\mathrm{cdf}}\) and Assumption~\ref{ass:soft-policy-admissibility}, the right-hand side is bounded uniformly over \((s,a)\in\mathcal S\times\mathcal A\). Therefore
\[
\sup_{(s,a)\in\mathcal S\times\mathcal A}
\int_{\mathbb R}
\Bigl(
(\mathcal T^{\pi}_{\mathrm{cdf}}F_{\mathcal Z})(s,a)(x)-\mathbbm 1\{x\ge0\}
\Bigr)^2\,dx
<\infty.
\]
In particular,
\[
(\mathcal T^{\pi}_{\mathrm{cdf}}F_{\mathcal Z})(s,a)\in\Gamma_F
\qquad
\text{for all }(s,a)\in\mathcal S\times\mathcal A.
\]
Moreover, the required measurability of the map
\[
(s,a)\mapsto
(\mathcal T^{\pi}_{\mathrm{cdf}}F_{\mathcal Z})(s,a)(x)
\]
for each fixed \(x\in\mathbb R\) is exactly the content of Assumption~\ref{ass:soft-policy-admissibility}. Therefore $\mathcal T^{\pi}_{\mathrm{cdf}}F_{\mathcal Z}\in\mathcal X^{\mathrm{cdf}},$
and hence $\mathcal T^{\pi}_{\mathrm{cdf}}:\mathcal X^{\mathrm{cdf}}\to\mathcal X^{\mathrm{cdf}}.$
\end{proof}

\subsection{Proof of Proposition~\ref{prop:soft-cdf-consistency}}
\label{app:proof-soft-cdf-consistency}
\begin{proof}
Fix \((s,a)\in\mathcal S\times\mathcal A\) and \(x\in\mathbb R\). By \eqref{eq:dist_soft_bellman_operator},
\[
(\mathcal T^\pi_{DS}\mathcal Z)(s,a)
=
\mathrm{Law}\!\left(
R(s,a)+\gamma\bigl(Z'-\alpha\log\pi(a'\mid s')\bigr)
\right),
\]
where \((s',a')\sim\mathcal P^\pi(\cdot\mid s,a)\) and, conditionally on \((s',a')\), the random variable \(Z'\) has law \(\mathcal Z(s',a')\). Therefore
\begin{align*}
F_{\mathcal T^\pi_{DS}\mathcal Z}(s,a)(x)
&=
\mathbb P\!\left(
R(s,a)+\gamma\bigl(Z'-\alpha\log\pi(a'\mid s')\bigr)\le x
\right)\\
&=
\mathbb E_{r\sim R(s,a),\,(s',a')\sim\mathcal P^\pi(\cdot\mid s,a)}
\Biggl[
\mathbb P\!\left(
Z'
\le
\frac{x-r}{\gamma}
+
\alpha\log\pi(a'\mid s')
\;\middle|\;
(s',a')
\right)
\Biggr].
\end{align*}
Since \(Z'\mid(s',a')\sim\mathcal Z(s',a')\), the inner probability equals
\[
F_{\mathcal Z}(s',a')
\Bigl(
\frac{x-r}{\gamma}
+
\alpha\log\pi(a'\mid s')
\Bigr).
\]
Substituting this into the previous expression yields
\[
F_{\mathcal T^\pi_{DS}\mathcal Z}(s,a)(x)
=
\mathbb E_{r\sim R(s,a),\,(s',a')\sim\mathcal P^\pi(\cdot\mid s,a)}
\Biggl[
F_{\mathcal Z}(s',a')
\Bigl(
\frac{x-r}{\gamma}
+
\alpha\log\pi(a'\mid s')
\Bigr)
\Biggr]
=
(\mathcal T^{\pi}_{\mathrm{cdf}}F_{\mathcal Z})(s,a)(x),
\]
which proves Proposition~\ref{prop:soft-cdf-consistency}.
\end{proof}

\subsection{Proof of Theorem~\ref{thm:soft-contraction}}
\label{app:proof-soft-contraction}
\begin{proof}
Fix \((s,a)\in\mathcal S\times\mathcal A\). By Definition~\ref{def:cdf-soft-bellman-operator},
\begin{align*}
&(\mathcal T^{\pi}_{\mathrm{cdf}}F_{\mathcal Z_1})(s,a)(x)
-
(\mathcal T^{\pi}_{\mathrm{cdf}}F_{\mathcal Z_2})(s,a)(x) \\
&=
\mathbb E_{r\sim R(s,a),\,(s',a')\sim\mathcal P^\pi(\cdot\mid s,a)}
\Biggl[
F_{\mathcal Z_1}(s',a')
\Bigl(
\frac{x-r}{\gamma}
+
\alpha\log\pi(a'\mid s')
\Bigr)
-
F_{\mathcal Z_2}(s',a')
\Bigl(
\frac{x-r}{\gamma}
+
\alpha\log\pi(a'\mid s')
\Bigr)
\Biggr].
\end{align*}
Therefore,
\begin{align*}
&d_C^2\bigl(
(\mathcal T^{\pi}_{\mathrm{cdf}}F_{\mathcal Z_1})(s,a),
(\mathcal T^{\pi}_{\mathrm{cdf}}F_{\mathcal Z_2})(s,a)
\bigr) \\
&=
\int_{\mathbb R}
\Biggl|
\mathbb E_{r\sim R(s,a),\,(s',a')\sim\mathcal P^\pi(\cdot\mid s,a)}
\Biggl[
F_{\mathcal Z_1}(s',a')
\Bigl(
\frac{x-r}{\gamma}
+
\alpha\log\pi(a'\mid s')
\Bigr)
-
F_{\mathcal Z_2}(s',a')
\Bigl(
\frac{x-r}{\gamma}
+
\alpha\log\pi(a'\mid s')
\Bigr)
\Biggr]
\Biggr|^2
dx.
\end{align*}
Applying Jensen's inequality to the inner expectation yields
\begin{align*}
&d_C^2\bigl(
(\mathcal T^{\pi}_{\mathrm{cdf}}F_{\mathcal Z_1})(s,a),
(\mathcal T^{\pi}_{\mathrm{cdf}}F_{\mathcal Z_2})(s,a)
\bigr) \\
&\le
\mathbb E_{r\sim R(s,a),\,(s',a')\sim\mathcal P^\pi(\cdot\mid s,a)}
\int_{\mathbb R}
\Biggl|
F_{\mathcal Z_1}(s',a')
\Bigl(
\frac{x-r}{\gamma}
+
\alpha\log\pi(a'\mid s')
\Bigr)
-
F_{\mathcal Z_2}(s',a')
\Bigl(
\frac{x-r}{\gamma}
+
\alpha\log\pi(a'\mid s')
\Bigr)
\Biggr|^2
dx.
\end{align*}
Now fix \(r\) and \((s',a')\), and make the change of variable
\[
y=
\frac{x-r}{\gamma}
+
\alpha\log\pi(a'\mid s'),
\qquad
dx=\gamma\,dy.
\]
Then
\begin{align*}
&\int_{\mathbb R}
\Biggl|
F_{\mathcal Z_1}(s',a')
\Bigl(
\frac{x-r}{\gamma}
+
\alpha\log\pi(a'\mid s')
\Bigr)
-
F_{\mathcal Z_2}(s',a')
\Bigl(
\frac{x-r}{\gamma}
+
\alpha\log\pi(a'\mid s')
\Bigr)
\Biggr|^2
dx \\
&=
\gamma
\int_{\mathbb R}
\bigl|
F_{\mathcal Z_1}(s',a')(y)-F_{\mathcal Z_2}(s',a')(y)
\bigr|^2dy \\
&=
\gamma\,d_C^2\bigl(F_{\mathcal Z_1}(s',a'),F_{\mathcal Z_2}(s',a')\bigr).
\end{align*}
Substituting this into the previous inequality gives
\begin{align*}
d_C^2\bigl(
(\mathcal T^{\pi}_{\mathrm{cdf}}F_{\mathcal Z_1})(s,a),
(\mathcal T^{\pi}_{\mathrm{cdf}}F_{\mathcal Z_2})(s,a)
\bigr)
&\le
\gamma\,
\mathbb E_{r\sim R(s,a),\,(s',a')\sim\mathcal P^\pi(\cdot\mid s,a)}
d_C^2\bigl(F_{\mathcal Z_1}(s',a'),F_{\mathcal Z_2}(s',a')\bigr) \\
&\le
\gamma\,d_{\mathrm{Cr}}^2(F_{\mathcal Z_1},F_{\mathcal Z_2}).
\end{align*}
Since \((s,a)\) was arbitrary, taking the supremum over \((s,a)\in\mathcal S\times\mathcal A\) and then square roots yields
\[
d_{\mathrm{Cr}}
\bigl(
\mathcal T^{\pi}_{\mathrm{cdf}}F_{\mathcal Z_1},
\mathcal T^{\pi}_{\mathrm{cdf}}F_{\mathcal Z_2}
\bigr)
\le
\sqrt{\gamma}\,d_{\mathrm{Cr}}(F_{\mathcal Z_1},F_{\mathcal Z_2}),
\]
which proves $d_{\mathrm{Cr}}
\bigl(
\mathcal T^{\pi}_{\mathrm{cdf}}F_{\mathcal Z_1},
\mathcal T^{\pi}_{\mathrm{cdf}}F_{\mathcal Z_2}
\bigr)
\le
\sqrt{\gamma}\,d_{\mathrm{Cr}}(F_{\mathcal Z_1},F_{\mathcal Z_2}).$
\end{proof}

\subsection{Proof of Proposition~\ref{prop:exact-evaluation-dspi}}
\label{appendix:dspi-proof}
\begin{proof}
The first claim follows directly from the evaluation theory developed in Sections~\ref{sec:cdf-soft-evaluation} and~\ref{sec:spectral-soft}. For each policy \(\pi_k\), this theory provides the evaluation objects \(\mathcal Z^{\pi_k}\), \(F_{\mathcal Z^{\pi_k}}\), and \(\nu^{\pi_k}\), together with the induced soft critic \(Q^{\pi_k}\). Defining \(\pi_{k+1}\) by \eqref{eq:soft-policy-improvement-step} therefore yields the stated distributional soft policy-iteration scheme with policy evaluation under the Cram\'er geometry.

The inequality
\[
Q^{\pi_{k+1}}(s,a)\ge Q^{\pi_k}(s,a),
\qquad
\forall (s,a)\in\mathcal S\times\mathcal A,
\]
is exactly the conclusion of Lemma~\ref{lem:soft-policy-improvement}.

Recall the soft state-value function defined in Section~\ref{sec:dspi-perspective}:
\[
V^{\pi_j}(s)
:=
\mathbb E_{a\sim\pi_j(\cdot\mid s)}
\bigl[
Q^{\pi_j}(s,a)-\alpha\log\pi_j(a\mid s)
\bigr],
\qquad j\in\{k,k+1\}.
\]
For every \(s\in\mathcal S\), the defining optimality of \(\pi_{k+1}\) in \eqref{eq:soft-policy-improvement-step} gives
\[
\mathbb E_{a\sim\pi_{k+1}(\cdot\mid s)}
\bigl[
Q^{\pi_k}(s,a)-\alpha\log\pi_{k+1}(a\mid s)
\bigr]
\ge
\mathbb E_{a\sim\pi_k(\cdot\mid s)}
\bigl[
Q^{\pi_k}(s,a)-\alpha\log\pi_k(a\mid s)
\bigr].
\]
Since \(Q^{\pi_{k+1}}(s,a)\ge Q^{\pi_k}(s,a)\) for all \((s,a)\), we also have
\[
\mathbb E_{a\sim\pi_{k+1}(\cdot\mid s)}
\bigl[
Q^{\pi_{k+1}}(s,a)-\alpha\log\pi_{k+1}(a\mid s)
\bigr]
\ge
\mathbb E_{a\sim\pi_{k+1}(\cdot\mid s)}
\bigl[
Q^{\pi_k}(s,a)-\alpha\log\pi_{k+1}(a\mid s)
\bigr].
\]
Combining the last two displays yields
\[
V^{\pi_{k+1}}(s)\ge V^{\pi_k}(s),
\qquad
\forall s\in\mathcal S.
\]
Finally, by \eqref{eq:maxent-objective-v-form},
\[
J(\pi_j;\rho)=\mathbb E_{S_t\sim\rho}\bigl[V^{\pi_j}(S_t)\bigr],
\qquad
j\in\{k,k+1\}.
\]
It follows that
\[
J(\pi_{k+1};\rho)
=
\mathbb E_{S_t\sim\rho}\bigl[V^{\pi_{k+1}}(S_t)\bigr]
\ge
\mathbb E_{S_t\sim\rho}\bigl[V^{\pi_k}(S_t)\bigr]
=
J(\pi_k;\rho).
\]
\end{proof}

\section{Further direct results on the spectral domain}
\label{app:spectral-consequences}

The core spectral statements in the main text concern transport of the distributional soft Bellman dynamics, together with the induced correspondence of fixed points and iterates. For completeness, we record here several further consequences of that transport when \(\mathcal X^{\mathrm{adm}}\) is equipped with the induced metric \(\widetilde d_{\mathrm{Cr}}\). These results are direct analogues of the CDF-level contraction, fixed-point, and iterate-convergence statements, but are not needed for the main line of analysis.

\subsection{Spectral contraction under the induced metric}
\label{app:spectral-contraction}

We first note that the CDF-level contraction property transports directly to the spectral domain once the induced metric is used.

\begin{proposition}[Spectral contraction under the induced metric]
\label{prop:spectral-soft-contraction}
Suppose Assumption~\ref{ass:soft-policy-admissibility} holds. Then \(\mathcal T^{\pi}\) is a \(\sqrt{\gamma}\)-contraction on \((\mathcal X^{\mathrm{adm}},\widetilde d_{\mathrm{Cr}})\). More precisely, for any \(\nu_1,\nu_2\in\mathcal X^{\mathrm{adm}}\),
\begin{equation}
\label{eq:spectral-soft-contraction}
\widetilde d_{\mathrm{Cr}}
\bigl(
\mathcal T^{\pi}\nu_1,
\mathcal T^{\pi}\nu_2
\bigr)
\le
\sqrt{\gamma}\,\widetilde d_{\mathrm{Cr}}(\nu_1,\nu_2).
\end{equation}
\end{proposition}

\begin{proof}
Let \(\nu_1,\nu_2\in\mathcal X^{\mathrm{adm}}\). Since \(\mathbb V:\mathcal X^{\mathrm{cdf}}\to\mathcal X^{\mathrm{adm}}\) is bijective, there exist unique \(F_{\mathcal Z_1},F_{\mathcal Z_2}\in\mathcal X^{\mathrm{cdf}}\) such that
\[
\nu_1=\mathbb V F_{\mathcal Z_1},
\qquad
\nu_2=\mathbb V F_{\mathcal Z_2}.
\]
By Definition~\ref{def:spectral-soft-bellman},
\[
\mathbb V^{-1}(\mathcal T^{\pi}\nu_1)
=
\mathcal T^{\pi}_{\mathrm{cdf}}F_{\mathcal Z_1},
\qquad
\mathbb V^{-1}(\mathcal T^{\pi}\nu_2)
=
\mathcal T^{\pi}_{\mathrm{cdf}}F_{\mathcal Z_2}.
\]
Therefore, by the definition of the induced metric \(\widetilde d_{\mathrm{Cr}}\),
\begin{align*}
\widetilde d_{\mathrm{Cr}}
\bigl(
\mathcal T^{\pi}\nu_1,
\mathcal T^{\pi}\nu_2
\bigr)
&=
d_{\mathrm{Cr}}
\bigl(
\mathbb V^{-1}(\mathcal T^{\pi}\nu_1),
\mathbb V^{-1}(\mathcal T^{\pi}\nu_2)
\bigr) \\
&=
d_{\mathrm{Cr}}
\bigl(
\mathcal T^{\pi}_{\mathrm{cdf}}F_{\mathcal Z_1},
\mathcal T^{\pi}_{\mathrm{cdf}}F_{\mathcal Z_2}
\bigr).
\end{align*}
Applying Theorem~\ref{thm:soft-contraction}, we obtain
\[
d_{\mathrm{Cr}}
\bigl(
\mathcal T^{\pi}_{\mathrm{cdf}}F_{\mathcal Z_1},
\mathcal T^{\pi}_{\mathrm{cdf}}F_{\mathcal Z_2}
\bigr)
\le
\sqrt{\gamma}\,d_{\mathrm{Cr}}(F_{\mathcal Z_1},F_{\mathcal Z_2}).
\]
Using again the definition of \(\widetilde d_{\mathrm{Cr}}\), this becomes
\[
\widetilde d_{\mathrm{Cr}}
\bigl(
\mathcal T^{\pi}\nu_1,
\mathcal T^{\pi}\nu_2
\bigr)
\le
\sqrt{\gamma}\,\widetilde d_{\mathrm{Cr}}(\nu_1,\nu_2),
\]
which proves \eqref{eq:spectral-soft-contraction}.
\end{proof}

This transported contraction property allows the Banach fixed-point argument to be reproduced on the spectral domain, once completeness of the induced metric space is noted.

\subsection{Spectral fixed-point uniqueness under the induced metric}
\label{app:spectral-fixed-point-uniqueness}

We next verify that the transported spectral space is complete under \(\widetilde d_{\mathrm{Cr}}\), and hence supports the same fixed-point argument as at the CDF level.

\begin{lemma}[Completeness of the transported spectral space]
\label{lem:spectral-space-complete}
The metric space \((\mathcal X^{\mathrm{adm}},\widetilde d_{\mathrm{Cr}})\) is complete.
\end{lemma}

\begin{proof}
Let \(\{\nu_n\}_{n\ge1}\subset \mathcal X^{\mathrm{adm}}\) be a Cauchy sequence with respect to \(\widetilde d_{\mathrm{Cr}}\). By definition of the induced metric,
\[
\widetilde d_{\mathrm{Cr}}(\nu_n,\nu_m)
=
d_{\mathrm{Cr}}(\mathbb V^{-1}\nu_n,\mathbb V^{-1}\nu_m)
\qquad\text{for all }n,m\ge1.
\]
Hence \(\{\mathbb V^{-1}\nu_n\}_{n\ge1}\subset \mathcal X^{\mathrm{cdf}}\) is a Cauchy sequence in \((\mathcal X^{\mathrm{cdf}},d_{\mathrm{Cr}})\). By Lemma~\ref{lem:cdf-space-complete}, there exists \(F\in\mathcal X^{\mathrm{cdf}}\) such that
\[
d_{\mathrm{Cr}}(\mathbb V^{-1}\nu_n,F)\to 0.
\]
Define
\[
\nu:=\mathbb V F\in\mathcal X^{\mathrm{adm}}.
\]
Then, again by the definition of \(\widetilde d_{\mathrm{Cr}}\),
\[
\widetilde d_{\mathrm{Cr}}(\nu_n,\nu)
=
d_{\mathrm{Cr}}(\mathbb V^{-1}\nu_n,\mathbb V^{-1}\nu)
=
d_{\mathrm{Cr}}(\mathbb V^{-1}\nu_n,F)\to 0.
\]
Thus \(\nu_n\to \nu\) in \((\mathcal X^{\mathrm{adm}},\widetilde d_{\mathrm{Cr}})\). Therefore \((\mathcal X^{\mathrm{adm}},\widetilde d_{\mathrm{Cr}})\) is complete.
\end{proof}

Combining completeness with Proposition~\ref{prop:spectral-soft-contraction} yields the spectral fixed-point result in Banach form.

\begin{corollary}[Unique spectral fixed point under the induced metric]
\label{cor:spectral-fixed-point-induced-metric}
Suppose Assumption~\ref{ass:soft-policy-admissibility} holds. Then \(\mathcal T^{\pi}\) admits a unique fixed point in \((\mathcal X^{\mathrm{adm}},\widetilde d_{\mathrm{Cr}})\). Equivalently, the element
\[
\nu^\pi=\mathbb V F_{\mathcal Z^\pi}
\]
from Corollary~\ref{cor:spectral-soft-fixed-point} is the unique fixed point of \(\mathcal T^{\pi}\) on \(\mathcal X^{\mathrm{adm}}\).
\end{corollary}

\begin{proof}
By Proposition~\ref{prop:spectral-soft-contraction}, \(\mathcal T^{\pi}\) is a contraction on \((\mathcal X^{\mathrm{adm}},\widetilde d_{\mathrm{Cr}})\). By Lemma~\ref{lem:spectral-space-complete}, this metric space is complete. The Banach fixed-point theorem therefore yields a unique fixed point of \(\mathcal T^{\pi}\) in \(\mathcal X^{\mathrm{adm}}\). By Corollary~\ref{cor:spectral-soft-fixed-point}, this fixed point is precisely \(\nu^\pi=\mathbb V F_{\mathcal Z^\pi}\).
\end{proof}

With the spectral fixed point thus identified, the convergence of spectral iterates follows in the standard way.

\subsection{Spectral iterate convergence under the induced metric}
\label{app:spectral-iterate-convergence}

We finally record the corresponding geometric convergence of the spectral soft policy evaluation iterates under the induced metric.

\begin{corollary}[Spectral iterate convergence under the induced metric]
\label{cor:spectral-iterate-induced-metric}
Suppose Assumption~\ref{ass:soft-policy-admissibility} holds. Let \(\nu_0\in\mathcal X^{\mathrm{adm}}\) be arbitrary, and define the spectral iterates by
\begin{equation}
\label{eq:appendix-spectral-iterate-definition}
\nu_{k+1}
=
\mathcal T^{\pi}\nu_k,
\qquad k\ge 0.
\end{equation}
Then \(\nu_k\to \nu^\pi\) in \((\mathcal X^{\mathrm{adm}},\widetilde d_{\mathrm{Cr}})\), where \(\nu^\pi\) is the unique fixed point from Corollary~\ref{cor:spectral-fixed-point-induced-metric}. More precisely, for every \(k\ge0\),
\begin{equation}
\label{eq:appendix-spectral-iterate-error-bound}
\widetilde d_{\mathrm{Cr}}(\nu_k,\nu^\pi)
\le
(\sqrt{\gamma})^k\,\widetilde d_{\mathrm{Cr}}(\nu_0,\nu^\pi).
\end{equation}
\end{corollary}

\begin{proof}
Applying Proposition~\ref{prop:spectral-soft-contraction} to \(\nu_k\) and \(\nu^\pi\) gives
\[
\widetilde d_{\mathrm{Cr}}
\bigl(
\mathcal T^{\pi}\nu_k,
\mathcal T^{\pi}\nu^\pi
\bigr)
\le
\sqrt{\gamma}\,\widetilde d_{\mathrm{Cr}}(\nu_k,\nu^\pi).
\]
Since \(\nu^\pi\) is a fixed point and \(\nu_{k+1}=\mathcal T^{\pi}\nu_k\), this becomes
\[
\widetilde d_{\mathrm{Cr}}(\nu_{k+1},\nu^\pi)
\le
\sqrt{\gamma}\,\widetilde d_{\mathrm{Cr}}(\nu_k,\nu^\pi).
\]
Iterating this inequality yields \eqref{eq:appendix-spectral-iterate-error-bound}. In particular,
\[
\widetilde d_{\mathrm{Cr}}(\nu_k,\nu^\pi)\to 0,
\]
so \(\nu_k\to \nu^\pi\) in \((\mathcal X^{\mathrm{adm}},\widetilde d_{\mathrm{Cr}})\).
\end{proof}







\end{document}